\def\year{2018}\relax
\documentclass[letterpaper]{article} 
\usepackage{aaai18}  
\usepackage{times}  
\usepackage{helvet}  
\usepackage{courier}  
\usepackage{url}  
\usepackage{graphicx}  

\usepackage{hyperref}

\usepackage{amsthm}
\newtheorem{definition}{Definition}
\newtheorem{proposition}{Proposition}

\begin{document}
%
\title{Smooth Inter-layer Propagation of Stabilized Neural Networks for Classification}
\author{Jingfeng Zhang \\
	School of Computing \\ 
	National University of Singapore \\
	j-zhang@comp.nus.edu.sg \\
	\And Laura Wynter\\
	IBM Research \\
	Singapore, 018983\\
	lwynter@sg.ibm.com
}
\maketitle
\begin{abstract}
Recent work has studied the reasons for the remarkable performance of deep neural networks in image classification. We examine    batch normalization on the one hand and the dynamical systems view of residual networks on the other hand. Our goal is in understanding the notions of stability and smoothness of the inter-layer propagation of ResNets so as to explain when they contribute to significantly  enhanced performance.  We postulate that such stability is of importance for the trained ResNet to transfer. 
\end{abstract}

\section{Introduction}
A significant improvement in the performance of neural networks was achieved with the introduction of Residual Networks (ResNets) by  \cite{he2016deep}. ResNets  were  among the first effective ultra deep convolutional networks. The authors observed that the depth of a neural network is of crucial importance to its  performance, but that empirical results consistently exhibited performance loss with increasing depths even via an identity mapping. The identity mapping in theory should have no detrimental effect on the network performance, but experience showed that it caused substantial degradation. This was thought not to simply be a result of vanishing or exploding gradients, as the authors employed batch normalization \cite{ioffe2015batch} to improve the stochastic gradient descent (SGD) convergence and yet witnessed the same performance loss.

The solution proposed by the authors to solve the issue of degradation with increasing depth was to define a new mapping, ${\cal F}(x)$,  as the residual of the original mapping ${\cal F}(x):={\cal H}(x)-x$,  leading to a re-defined mapping  ${\cal F}(x)+x$. This  leads  to  the implementation of skip connections in the network, in which a group of layers, called a block, is skipped by performing an identity mapping. Since its introduction and due to the remarkable results achieved with ResNets, a number of variants have been introduced, from ResNeXt, \cite{resnext} which splits the blocks into branches and adds a new tune-able cardinality parameter to determine the number of branches, to Densely Connected Nets \cite{Huang2017DenselyCC}, which connect feature maps of each block to all subsequent blocks.
We shall however take the original ResNet as our starting point.

Shortly before the introduction of ResNets, batch normalization was introduced by  \cite{ioffe2015batch}.  Batch normalization was employed by \cite{he2016deep} in their definition of ResNets and indeed  many if not most CNN implementations today make use of batch normalization. The goal of batch normalization according to the authors was to accelerate the model training process by allowing for effective training with higher learning rates, as well as reducing the dependence on  parameter initialization. They accomplish this by addressing and correcting the so-called internal covariate
shift so as to control and normalize the input distributions of each mini-batch in the network.

\cite{2018arXiv180511604S} postulate that the well-documented improvement of batch normalization on the accuracy and training rate of neural networks are due less, or, in fact, not at all to its ability to control the first two moments of the mini-batch input distributions than to another phenomenon. The authors claim  that batch norm does not reduce internal covariate shift but rather
increases the smoothness of the optimization problem, and specifically  reduces  the bound on the (local) Lipschitz constant of of the loss and the (local) $\beta$-smoothness constant of the gradients. The authors demonstrate this new explanation of batch norm's success convincingly, both theoretically and empirically.

Empirically, the authors show  that adding noise with non-zero-mean and non-unit-variance  to the activations in each batch, after each batch normalization procedure, thus increasing the covariate shift from its original level, still results in marked improvement in training accuracy. On the other hand, the authors show experimentally that the objective function and the gradients are considerably more smooth with batch norm than without.

In parallel, recent work has examined neural networks, and specifically ResNets, as discrete (homogeneous, time invariant) dynamical systems \cite{Liao2016BridgingTG}.  \cite{E2017} described the ResNet update formula as a step in a discrete forward Euler method for solving an ordinary differential equation. In an impressive series of papers, \cite{Haber17}, \cite{haber17b}, \cite{haber18}, \cite{haber18b}, \cite{haber18c} have taken this interpretation of ResNet to the next level by  generalizing the architecture using advances in numerical methods for the resolution of partial differential equations, from update formulae for parabolic and hyperbolic PDE, such as using second-order dynamics,  to multi-scale methods. 

A crucial purpose of these studies is to propose architectures that are stable in the sense of a dynamical system. The analogy is that the ResNet forward propagation should take the initial data from a state in which they are not linearly classifiable, to one in which they are, with a smooth trajectory. The authors claim that this smoothness is important, as it is in dynamical systems, to ensuring the insensitivity of the system to perturbations in the input data, be it from noise or adversarial attacks. Specifically, the authors propose that the forward propagation is well-posed and thus stable when the (real part of the) eigenvalues of the Jacobian of the nonlinear mapping in the ResNet update are non-positive.

We are interested in examining these different notions of stability and smoothness of the inter-layer propagation of ResNets so as to understand when they contribute to overall improved performance.  
The contributions of this paper are: (i) a detailed demonstration of the inter-layer propagation smoothness of these methods, and (ii) the introduction of a hybrid approach allowing for enhanced accuracy.

\section{Stability of ResNets}

Batch normalization and the interpretation of ResNets as a dynamical system and a step of the forward Euler method have made valuable contributions  towards   improving the stability of the original ResNet architecture.

\subsection{Stable Feature Transformation}

Consider first the dynamical system view of  ResNet. The aim of the neural network is to learn and simulate the process of underlying nonlinear feature transformations, so that the transformed features can be matched with target values, such as  categorical labels for classification and continuous quantities for regression. Specifically, the goal of the deep neural network is to take the initial data to a final state which is linearly separable.

As described in the references above, these underlying feature transformations can be characterized as equations of motions for the features themselves. In particular, we may view these transformations as a system of first order differential equations.

Such a system is formulated as 
\begin{equation}
	\dot{{y}} = f(t, {y} )  
\end{equation}
in which $f: I \times D  \rightarrow R^d$, $I \in R, D \in R^d$ and $t \geq 0 $.
${y}$ is the feature vector. 

The transition of an instance is thus a first order equation, together with its initial data, giving rise to the initial value problem (IVP):
\begin{equation}
	\dot{{y}} = f(t, {y} ) ,\  {y}(0) = {y}_{0} 
\end{equation}
where ${y}_{0} $ is the initial input vector of the instance.

It has been shown in \cite{E2017}, \cite{Haber17} that its forward propagation is an approximation of the solution trajectory of the discrete initial value problem. 

For a specific problem, we assume the underlying feature transformations are stable. 
In other words, we assume our partial differential equations (PDE), which characterize the the underlying feature transformations, are stable. 
Thus when we solve the initial value problem,  small changes in the initial data ${y}_{0}$ will result in  small changes in the transformed features at every transition step $t$. 

\subsubsection{IVP Stability}

Consider the perturbed IVP, 
\begin{equation}
	\dot{{y}} = f(t, {y}) ,\ {y}_{\epsilon}(0) = {y}_{0} +\epsilon, 
\end{equation} 
we have 
\begin{equation}
	\max_{t_0 \leq t} |{y}_{\epsilon}(t) - {y}(t) | \leq c\epsilon
\end{equation} for some constant $c \geq 0$.

In order to solve the stable IVP, we would like our numerical methods to be stable as well. 

One of those numerical methods is the explicit Euler method, which approximates a PDE by 
\begin{equation}
	{y}_{i+1} = {y}_{i} + h {f}(t_i, {y}_i)
	\label{eq:h}
\end{equation} 
where ${y}_{i}$ is the approximation of ${y}(t_i)$.

\begin{definition}[ResNet forward propagation step]
	Let  $h = 1$ and ${f}(t_i, {y}_i) ={K_2}_i \sigma ({K_1}_i {y}_i + {b_1}_i) + {b_2}_i$. 
\end{definition}
Then, the Euler method of \ref{eq:h} can be seen as the ResNet forward propagation step. We can thus state a necessary condition for its stability.

\begin{proposition}[Stability]
	The explicit Euler method,  given ${y}_0$, is stable if $h$ is sufficiently small, and satisfies 
	\begin{equation}
		\max |1 + h \lambda ({J}_i)| \leq 1
		\label{eq:eig}
	\end{equation}
	where $\lambda ({J}_i)$ is the set of eigenvalues of the Jacobian matrix  of ${f}(t_i, {y}_i)$ with respect to the features ${y}_i$, at the $i$th approximation step, and the maximum is taken over the real part  of the eigenvalues in the set.
\end{proposition}
\begin{proof}
	See, e.g.  \cite{ascher2008numerical} for detailed proof. 
\end{proof}

ResNet is not guaranteed to have smooth transitions between layers, which will lead to volatile feature transitions. The impact of unstable transitions has been postulated by \cite{Haber17}, \cite{haber18} as leading to potential loss of accuracy when input data is noisy or corrupted by adversaries.
We are interested in the use of ResNets beyond image classification.  The stability can be of potentially significant importance for the ability of the trained ResNet to transfer to other data sets as well as to  time-evolving data.

To illustrate the stability of the  convergence of the original formulation of  ResNet, consider the synthetic   data generated to represent a nonlinearly-separable binary classification problem in Figure \ref{fig:moon0}. 

Animations of feature propagation with scaled and fixed axes corresponding to Fig. \ref{fig:moonResNet} - Fig. \ref{fig:moon_data_res_batch_norm-0.1} can be found in the 
\href{https://www.dropbox.com/sh/majyu9d69r1y1ya/AAD-fJ-KmbK9-eHQRn_7d8PEa?dl=0}{supplement materials}.

\begin{figure}[h]
	\centering
	\includegraphics[scale = 0.25]{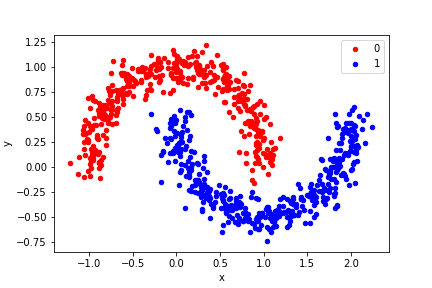} 
	\includegraphics[scale = 0.25]{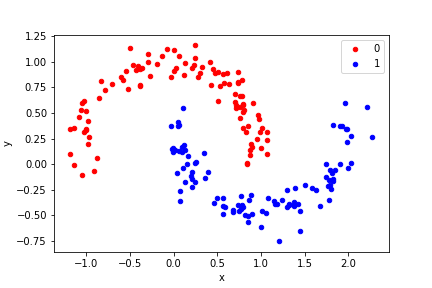}
	\caption{Synthetic nonlinearly-separable binary data: training data (left) and test data (right). }
	\label{fig:moon0}
\end{figure}

Consider now the plain vanilla ResNet (without batch normalization) on the 2-dimensional synthetic binary dataset in Fig. \ref{fig:moonResNet}, where the figures shown are selected from the 100 layers of the network: the 0th layer, and layers 5,10,20,\ldots, 100. Note that the axes ranges increase and shift over the 12 images; this is necessary for visualization as the activations diverge to the very large values on the deeper layers. 

It is instructive via this  2-dimensional classification problem to visually examine the convergence progression of the ResNet;   observe in Fig. \ref{fig:moonResNet}  that the initial, nonlinearly-separable binary features do become linearly aligned, which is indeed one of the goals of the network.   However, they do not become  well-separated by their binary values;   the regions of red and blue remain mixed over the line to which they tend, and the solutions  oscillate reaching an accuracy of 0.8. Animations of the convergence over the full 100 layers are available in the supplemental materials.

\begin{figure}[h]
	
	\includegraphics[scale = 0.18]{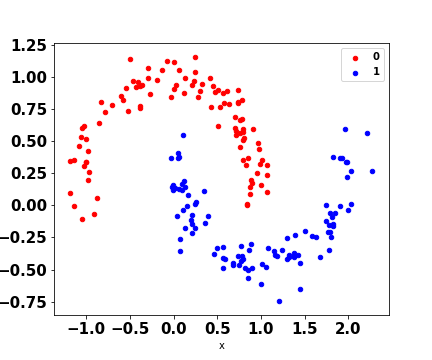}
	\includegraphics[scale = 0.18]{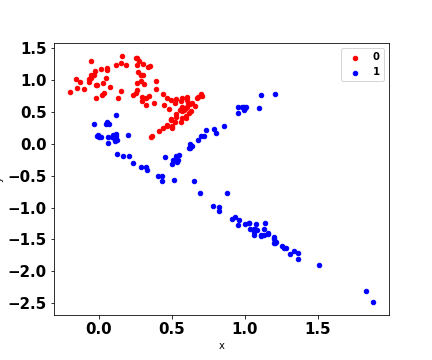}
	\includegraphics[scale = 0.18]{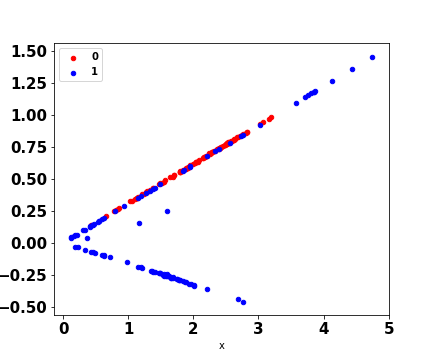}\\
	\includegraphics[scale = 0.18]{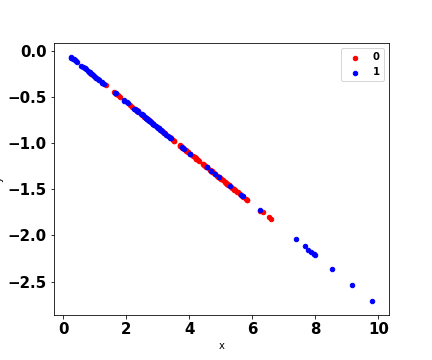} 
	\includegraphics[scale = 0.18]{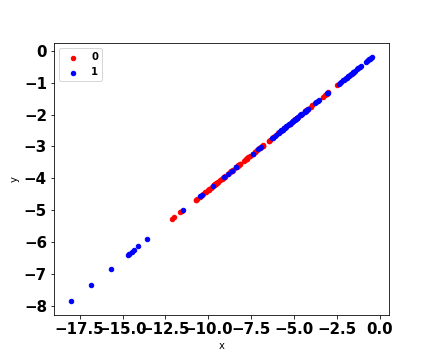}
	\includegraphics[scale = 0.18]{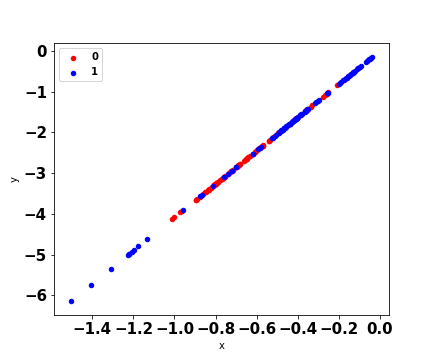}\\
	\includegraphics[scale = 0.18]{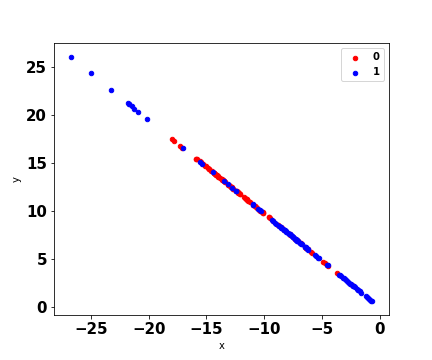}
	\includegraphics[scale = 0.18]{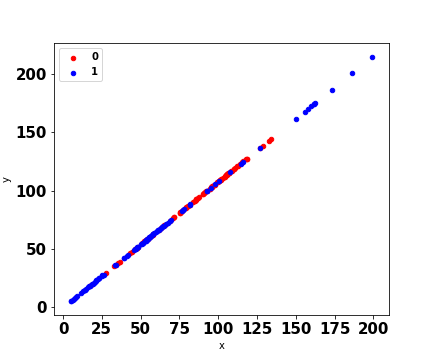}
	\includegraphics[scale = 0.18]{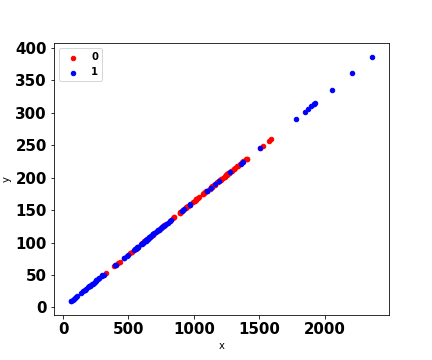}\\
	\includegraphics[scale = 0.18]{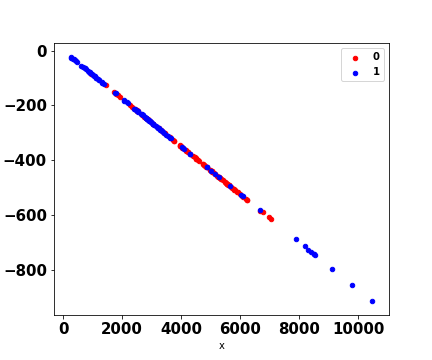}
	\includegraphics[scale = 0.18]{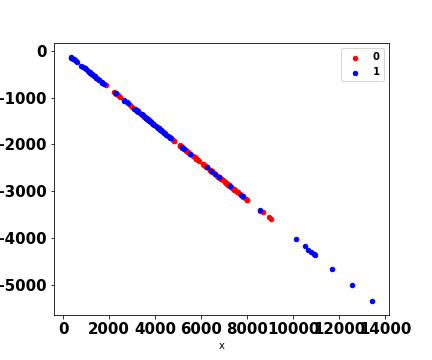}
	\includegraphics[scale = 0.18]{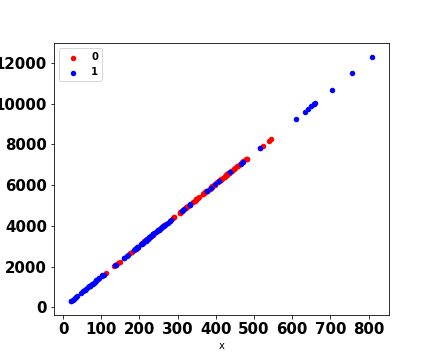}
	\caption{Plain vanilla ResNet (no batch normalization) on the synthetic 2-dim. binary data of Fig. \ref{fig:moon0}. Layers shown are 0,5,10,20,30\ldots 100. Axes ranges increase and shift  to visualize  diverging features with increasing depth.}
	\label{fig:moonResNet}
	
\end{figure}

\subsubsection{Stabilizing the ResNet using properties of the IVP}
Based on properties of the initial value problem, we consider two mechanisms to stabilize the transitions between layers and induce smoother algorithm convergence. 

\paragraph{Explicit step size parameter} As noted by \cite{haber18b}, the parameter $h$ in (\ref{eq:h}) can be viewed as an explicit step size in a numerical optimization iteration. The authors decrease  the explicit step size parameter in their multi-level approach  by half at each level  and double the number of residual blocks, thereby leaving unchanged the original differential equation. In our case, we treat the step size $h$ as an explicit parameter to be tuned so as to avoid divergence of the trajectory. 
This leads to the following update:

\begin{definition}[ResNet with explicit step size parameter] Let
	\begin{equation}
		{y}_{i+1} = {y}_{i} + h ( {K_2}_i \sigma ({K_1}_i {y}_i + {b_1}_i) + {b_2}_i ) \label{res_transition_block}
	\end{equation}
	where $h\in (0,1]$. 
\end{definition}
Note that setting $h=1$ we retrieve the plain vanilla ResNet.

\paragraph{Shrinkage blocks} Secondly, following \cite{haber18}, and in accordance with the theory of stable iterative approximation methods for PDEs, we 
add a specific constraint on the transition matrices, which has the purpose of forcing  the shrinkage of the (real parts of the) eigenvalues $\lambda ({J}_i)$, as given by the sufficient condition  for smooth convergence and stability (\ref{eq:eig}).
This leads to the following update.

\begin{definition}[Shrinkage block] The ResNet shrinkage block update is given by:
	\begin{equation}
		{y}_{i+1} = {y}_{i} - h ( {{K_1}_i } ^T \sigma ({K_1}_i {y}_i + {b_1}_i) + {b_2}_i ).  \label{stable_transition_block}
	\end{equation}
\end{definition}

\begin{proposition}
	Given the shrinkage block update of  (\ref{stable_transition_block}) , for all $h\in(0,1]$,  the eigenvalues of ${J}_i$ are real and bounded by zero.
\end{proposition}
\begin{proof}
	In the case of ResNet, we have that ${J}_i = -  ( {K_1}_{i} ^T \sigma ' ({K_1}_i {y}_i + {b_1}_i) {K_1}_i )$.
	As the activation function $\sigma$ is   non-decreasing, we have that $\sigma ' \geq 0 $. Thus the eigenvalues  $\lambda_{{J}_i}$ are real and bounded from above by zero for all $h\in (0,1]$. 
\end{proof}

Consider  the same synthetic 2-dim. binary classification data of Fig. \ref{fig:moon0}. As noted by \cite{haber18b}, reducing the step size parameter to a (fixed) $h=0.1$ would require an increase in the number of layers by 10 times to obtain the same differential equation. In practice, such dramatic increase in depth is not needed, and in many cases little to no increase in depth is required to achieve improved stability and accuracy. 
We illustrate the remarkable benefit of a small step size in Fig. \ref{fig:moon_data_res_h=0.1} on an extended network with a depth of 500 layers. We show the output at layers 0 and 5 as before and subsequently at layers 50,100, 150, 200,\ldots 500.
As expected, the algorithm progresses more slowly to a linear representation of the features, but it is also remarkably smoother and better.  With a 500-layer network the separation and linearization is perfect. Towards the deeper layers, while linearization was not complete, accuracy reaches 1.0 before layer 200.

\begin{figure}[h]
	
	\includegraphics[scale = 0.18]{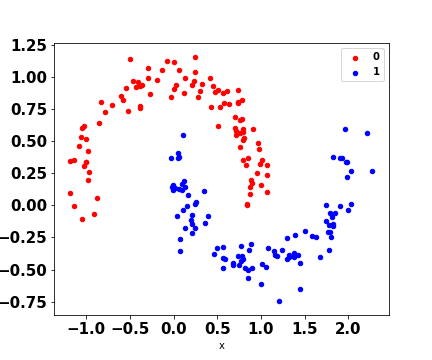}
	\includegraphics[scale = 0.18]{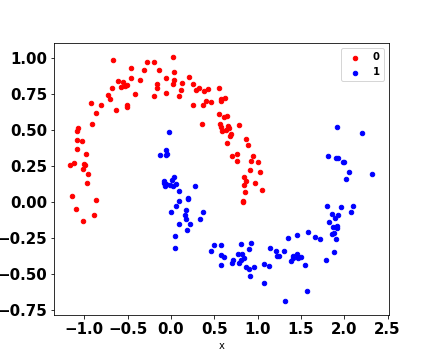}
	\includegraphics[scale = 0.18]{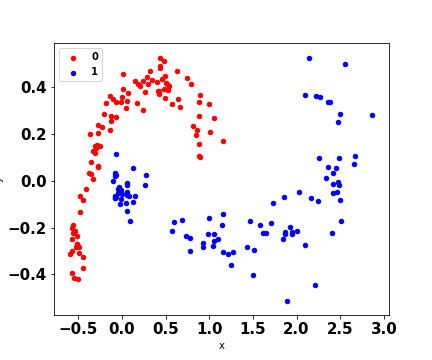} \\
	\includegraphics[scale = 0.18]{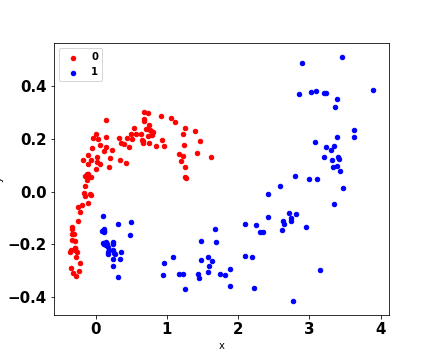} 
	\includegraphics[scale = 0.18]{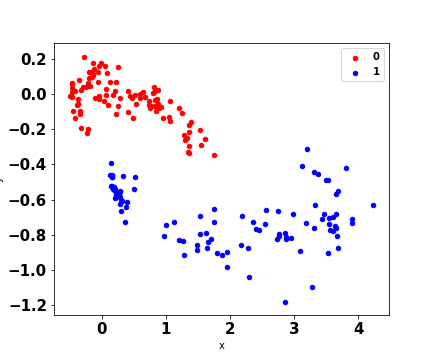}
	\includegraphics[scale = 0.18]{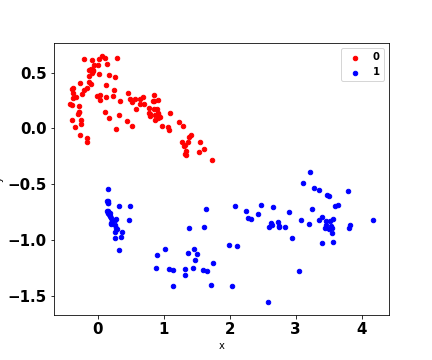} \\
	\includegraphics[scale = 0.18]{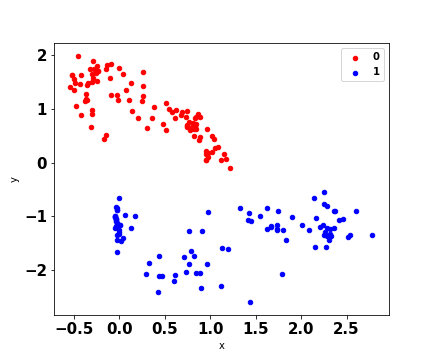}
	\includegraphics[scale = 0.18]{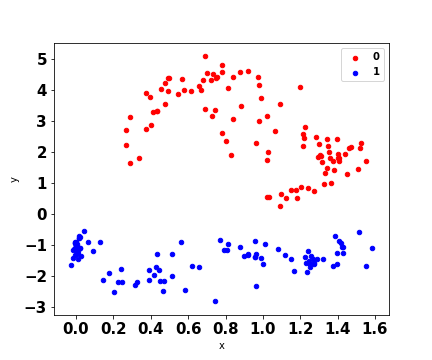} 
	\includegraphics[scale = 0.18]{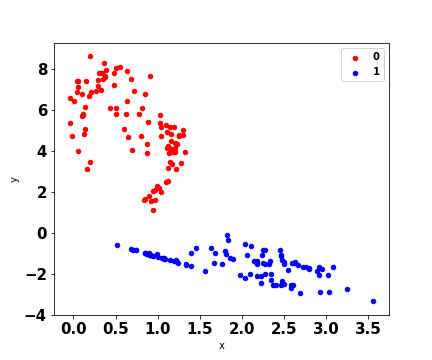} \\
	\includegraphics[scale = 0.18]{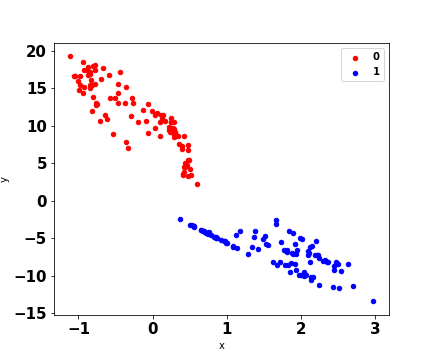}
	\includegraphics[scale = 0.18]{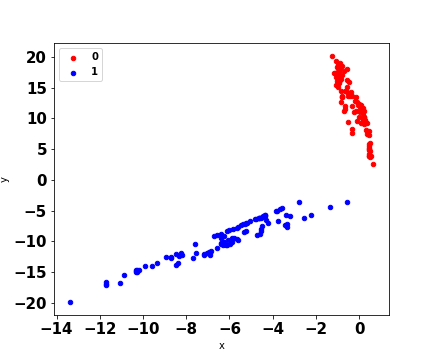}
	\includegraphics[scale = 0.18]{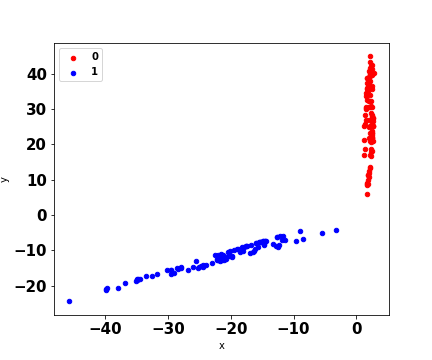}
	\caption{Plain ResNet with explicit step size of  $h = 0.1$ on the synthetic 2-dim. data of Fig. \ref{fig:moon0} and a deeper (500-layer) network. Layers shown are 0,5,50,100,150 200,\ldots 500.
		Axes ranges shift  to better visualize    features with increasing depth; note that the divergence of Fig. \ref{fig:moonResNet} no longer occurs.}
	\label{fig:moon_data_res_h=0.1}
	
\end{figure}

Consider now the use of the shrinkage block on the ResNet, shown in Fig. \ref{fig:moon_data_shrinking_h=1.0}. Here, we keep the step size parameter $h=1$ as in the plain vanilla ResNet.
It is clear, given the properties of the update step, that although the stable shrinkage  block update  guarantees  every transition is stable,  when $h= 1 $ the method  shrinks the features too fast, observable from the scale of the axes in Fig. \ref{fig:moon_data_shrinking_h=1.0}.
Contrary to offering the linear separability of our features that we seek, the result in this case is that
the transited features become linear but less distinguishable at the deeper layers. Increased depth does not improve accuracy, which does not reach 0.7.

\begin{figure}[h]
	
	\includegraphics[scale = 0.18]{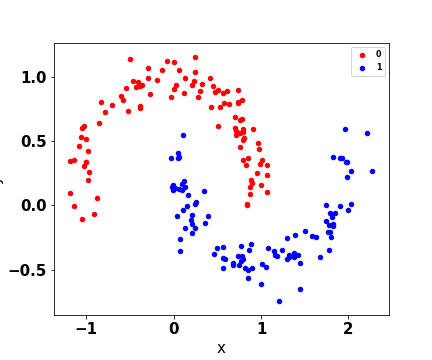}
	\includegraphics[scale = 0.18]{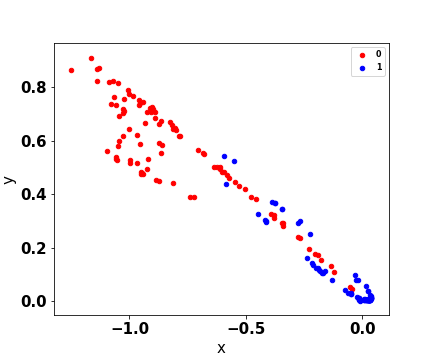}
	\includegraphics[scale = 0.18]{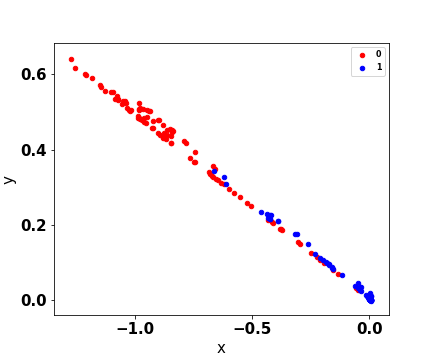}\\
	\includegraphics[scale = 0.18]{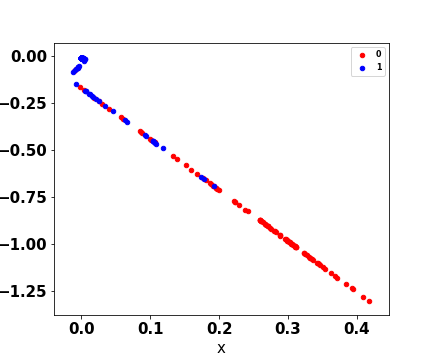} 
	\includegraphics[scale = 0.18]{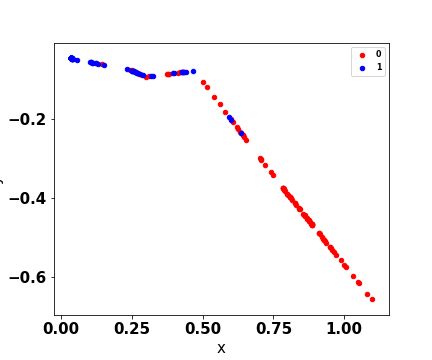}
	\includegraphics[scale = 0.18]{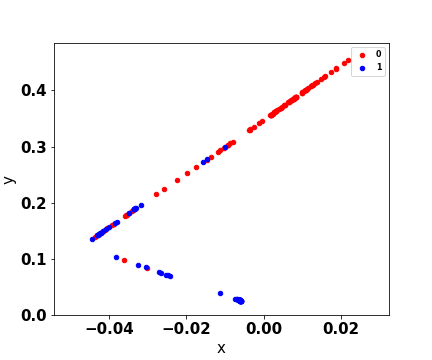}\\
	\includegraphics[scale = 0.18]{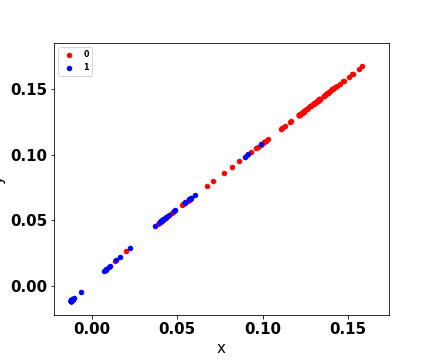}
	\includegraphics[scale = 0.18]{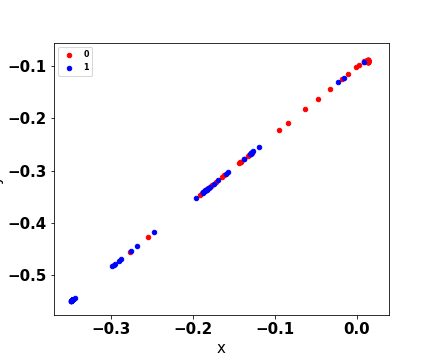}
	\includegraphics[scale = 0.18]{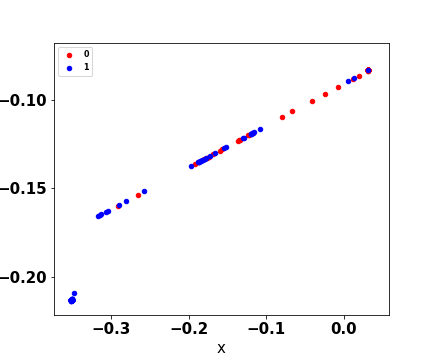}\\
	\includegraphics[scale = 0.18]{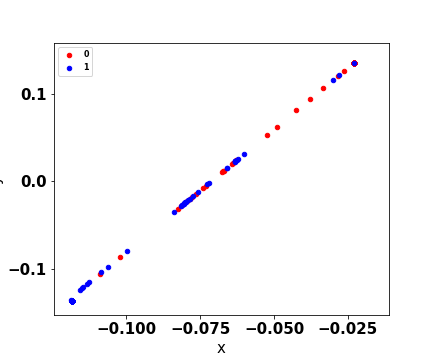}
	\includegraphics[scale = 0.18]{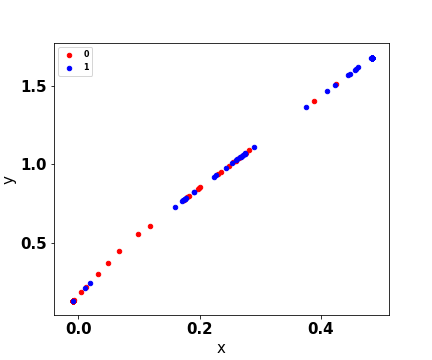}
	\includegraphics[scale = 0.18]{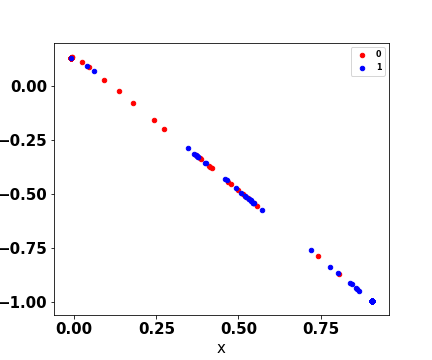}
	\caption{ResNet with shrinkage blocks on the synthetic 2-dim. binary data of Fig. \ref{fig:moon0}. Layers shown are 0,5,10,20,30\ldots 100. Axes  shift  to  visualize    features with increasing depth; note that  axes shrink rather than diverge.}
	\label{fig:moon_data_shrinking_h=1.0}
	
\end{figure}

Consider next the shrinkage block in which we reduce the step size to $h=0.1$,  in Fig. \ref{fig:moon_data_shrinking_h_0.1}.  The algorithm progression is as expected slower and smoother, and the resulting separation of  features is better though not perfect. Increased depth does not offer an advantage; accuracy on a 100-layer and 500-layer network are largely the same, at 0.85.

\begin{figure}[h]
	
	\includegraphics[scale = 0.18]{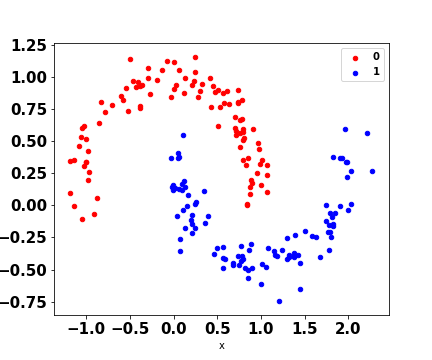}
	\includegraphics[scale = 0.18]{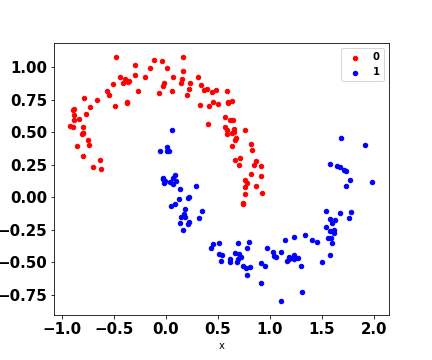}
	\includegraphics[scale = 0.18]{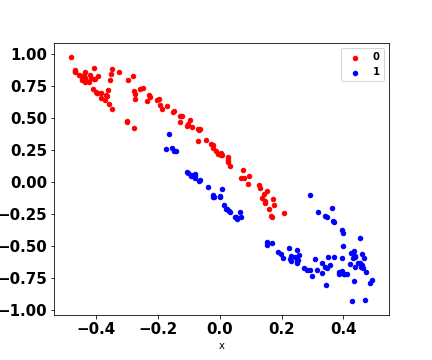} \\
	\includegraphics[scale = 0.18]{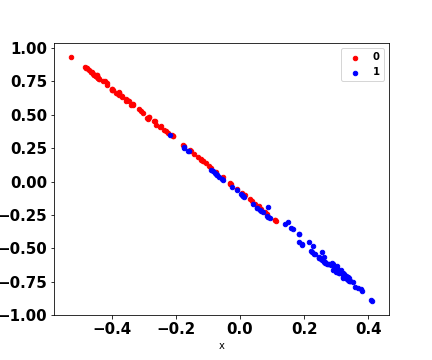}  
	\includegraphics[scale = 0.18]{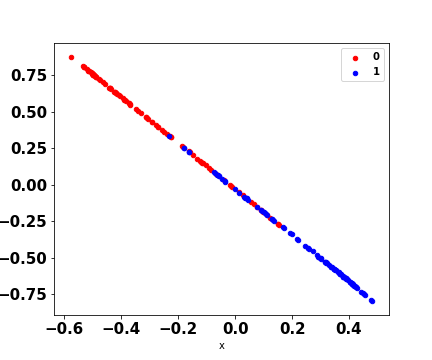}
	\includegraphics[scale = 0.18]{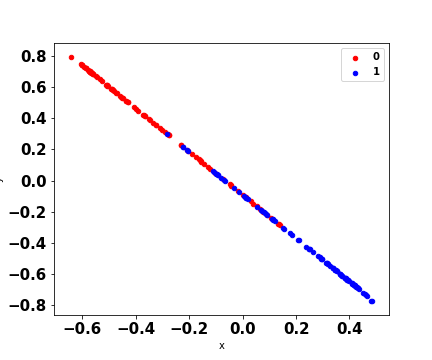} \\
	\includegraphics[scale = 0.18]{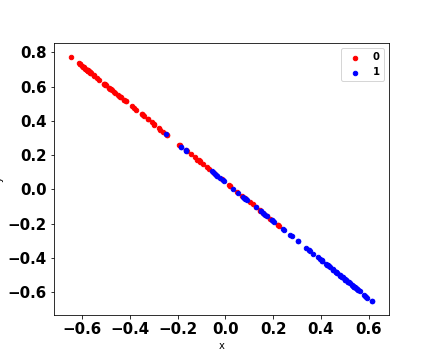}
	\includegraphics[scale = 0.18]{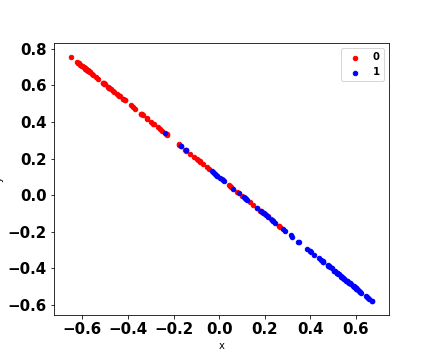} 
	\includegraphics[scale = 0.18]{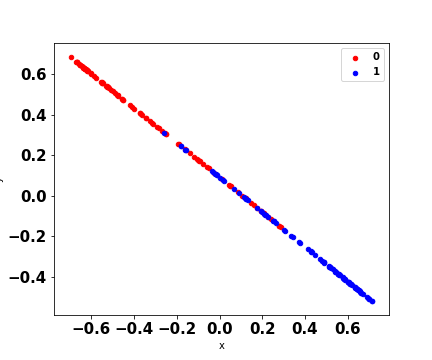} \\
	\includegraphics[scale = 0.18]{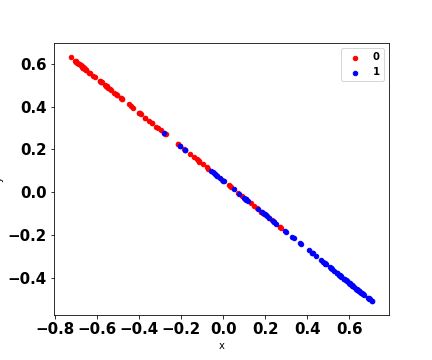}
	\includegraphics[scale = 0.18]{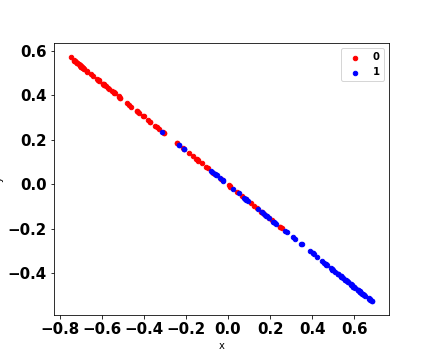}
	\includegraphics[scale = 0.18]{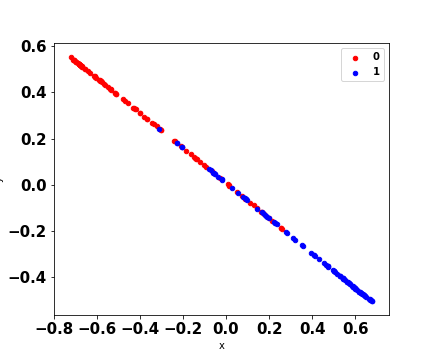}
	\caption{ResNet with shrinkage blocks and step size  $h=0.1$ on the synthetic 2-dim. binary data of Fig. \ref{fig:moon0} with a deeper (500 layer) network. Layers shown are 0,5,50,100,150 200,\ldots 500. Axes shift  to  visualize    features with increasing depth; note that axes ranges shrink.}
	\label{fig:moon_data_shrinking_h_0.1}
	
\end{figure}

\subsection{Batch normalization} 
Batch normalization \cite{ioffe2015batch} 
is a technique designed to  provide each mini-batch   in a neural network with inputs normalized to   zero mean - unit variance, and subsequently scaled for each mini-batch by an optimized linear scaling as a function of two parameters, a slope and an intercept. Batch normalization has been accepted as a gold standard in  most CNN implementations today.

\cite{2018arXiv180511604S} demonstrated theoretically and empirically that while batch normalization  offers greatly improved convergence and accuracy, it does not do so by correcting the internal covariate shift. They showed theoretically that the batch norm update decreases the (local) Lipschitz constant of the objective function and the (local) $\beta$-smoothness coefficient of its gradient.
Empirically, the authors examined  VGG networks and showed that batch norm does not in fact reduce internal covariate shift, and that when the internal covariate shift is increased (by adding noise from a distribution having non-zero-mean and non-unit-variance) the benefits of batch normalization still accrue. The authors illustrate the smoothness offered by batch norm on the objective function  and on the gradients.

We illustrate ResNet with batch norm on  the 2-dim. synthetic binary data in Fig. \ref{fig:moon_data_res_batch_norm}. The smoothness offered by batch norm is not apparent from the pattern as was the case with a reduced step size $h$ and the shrinkage block. However, some form of improved stability can be seen from the axes ranges  which do not explode as did those of  plain vanilla ResNet. The promise of higher accuracy however does not occur here from the use of batch norm, which does not reach the accuracy of plain vanilla ResNet.

We thus turn to a hybridization and show that batch normalized ResNet smoothness can be improved through the use of methods for improving stability of  numerical methods for ODEs. Fig. \ref{fig:moon_data_res_batch_norm-0.1} shows batch norm with a reduced step size of $h=0.1$. We illustrate the results on a deeper, 500-layer network. The accuracy reaches 1.0 by the 250th layer;  smoothness increases up through the 500th layer.

\begin{figure}[h]
	
	\includegraphics[scale = 0.18]{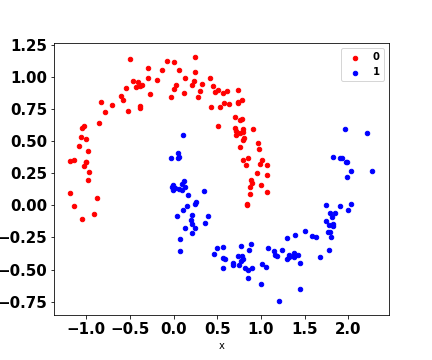}
	\includegraphics[scale = 0.18]{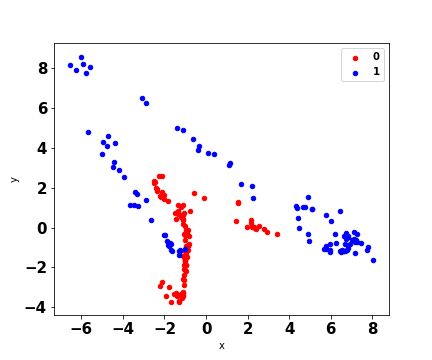}
	\includegraphics[scale = 0.18]{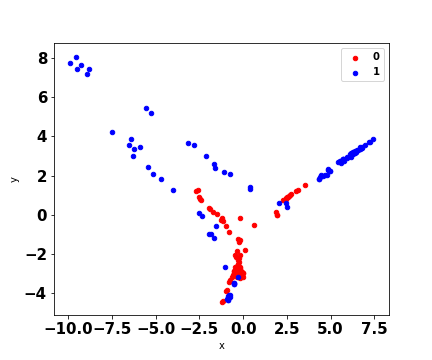}\\
	\includegraphics[scale = 0.18]{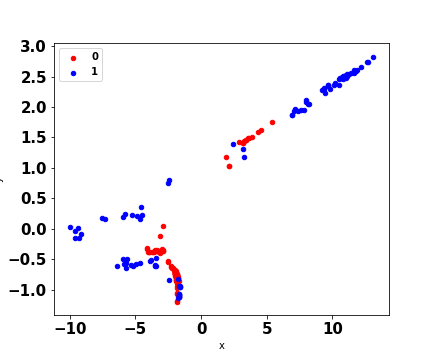} 
	\includegraphics[scale = 0.18]{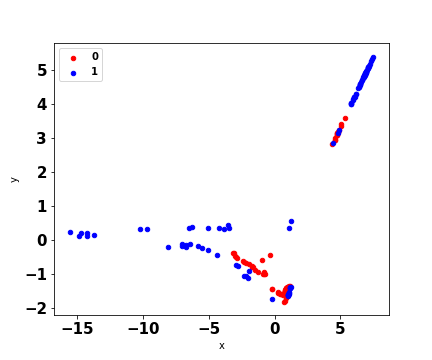}
	\includegraphics[scale = 0.18]{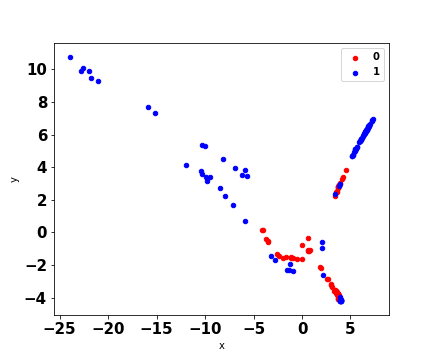}\\
	\includegraphics[scale = 0.18]{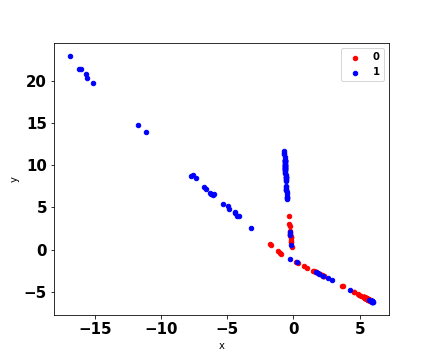}
	\includegraphics[scale = 0.18]{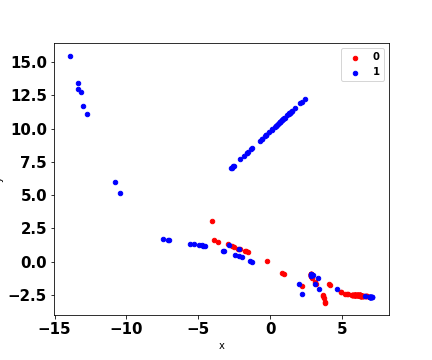} 
	\includegraphics[scale = 0.18]{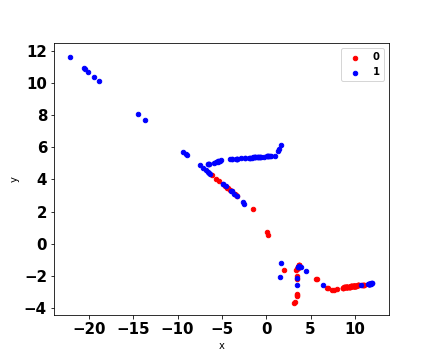}\\
	\includegraphics[scale = 0.18]{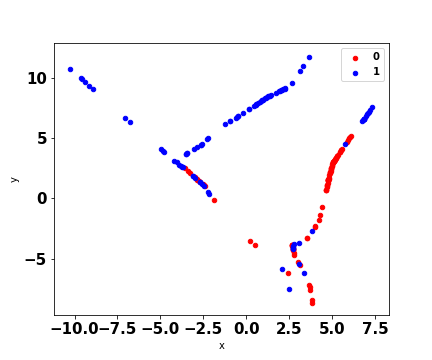}
	\includegraphics[scale = 0.18]{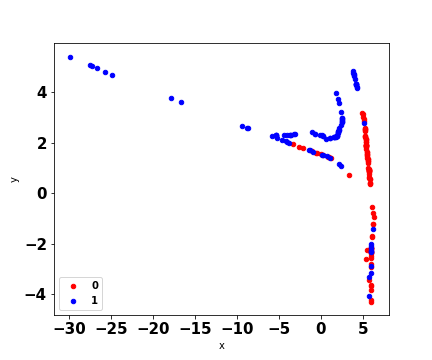}
	\includegraphics[scale = 0.18]{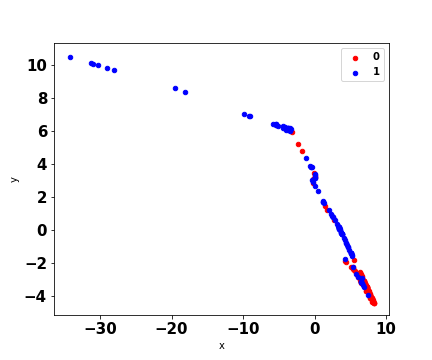}
	\caption{ResNet with batch normalization on the synthetic 2-dim. binary data of Fig. \ref{fig:moon0}. Layers shown are 0,5,10,20,30\ldots 100. Axes ranges shift  to better visualize    features with increasing depth; note that the axes do not diverge.}
	\label{fig:moon_data_res_batch_norm}
	
\end{figure}

\begin{figure}[h]
	
	\includegraphics[scale = 0.18]{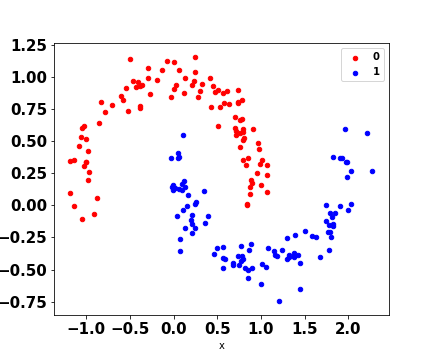}
	\includegraphics[scale = 0.18]{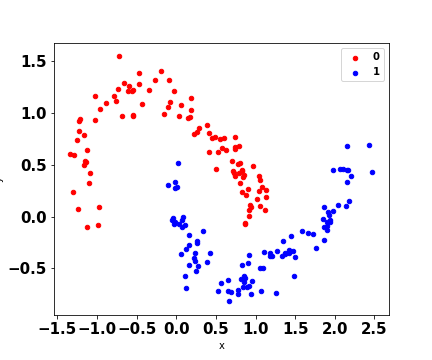}
	\includegraphics[scale = 0.18]{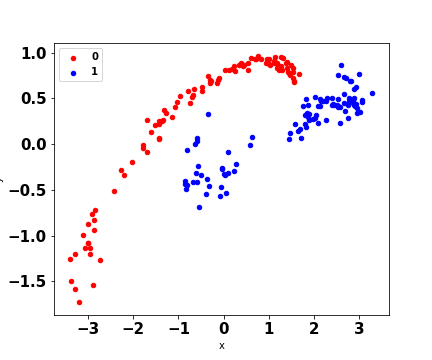}\\
	\includegraphics[scale = 0.18]{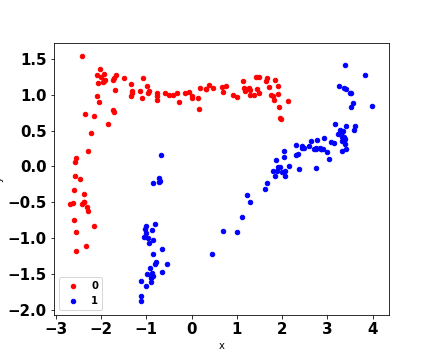} 
	\includegraphics[scale = 0.18]{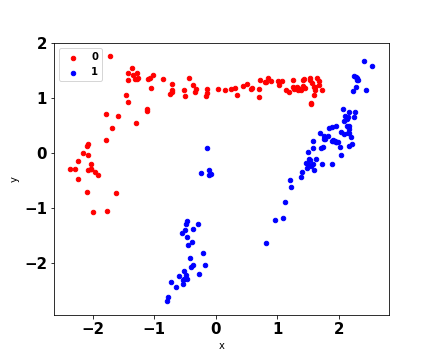}
	\includegraphics[scale = 0.18]{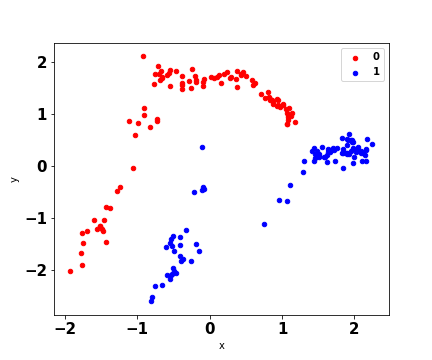}\\
	\includegraphics[scale = 0.18]{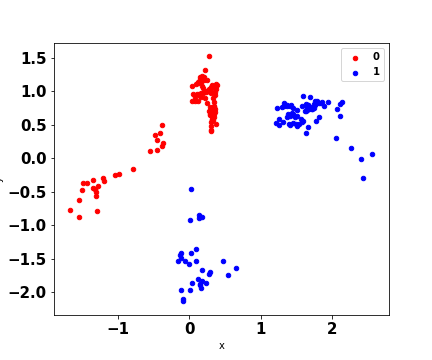}
	\includegraphics[scale = 0.18]{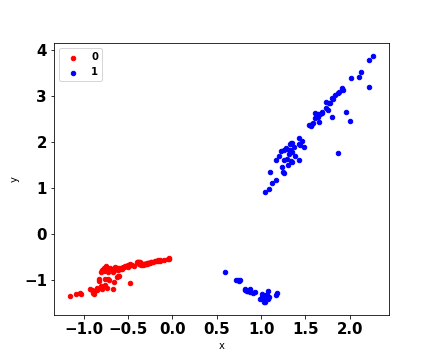} 
	\includegraphics[scale = 0.18]{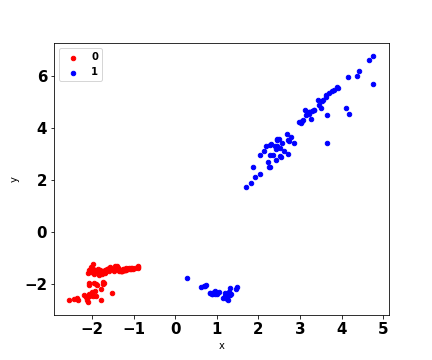}\\
	\includegraphics[scale = 0.18]{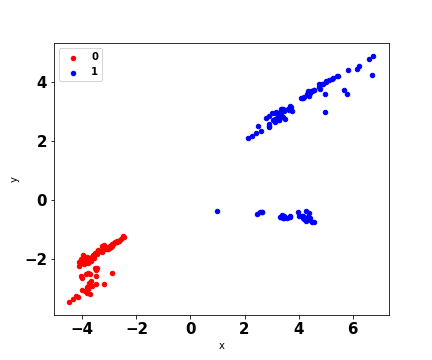}
	\includegraphics[scale = 0.18]{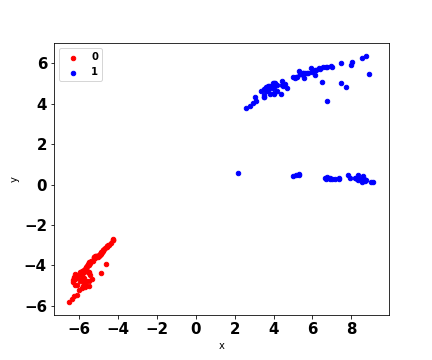}
	\includegraphics[scale = 0.18]{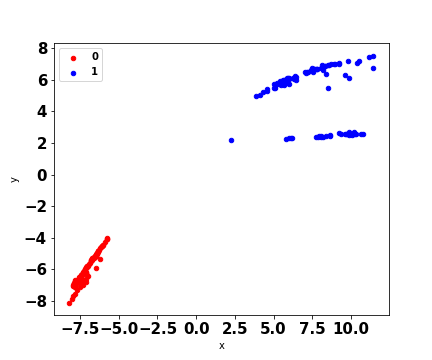}
	\caption{ResNet with batch norm and h = 0.1 on the synthetic 2-dim.  data of Fig. \ref{fig:moon0} with a deeper (500 layer) network. Layers shown are 0,5,50,100,150 200,\ldots 500. Axes ranges shift  to visualize features with increasing depth.}
	\label{fig:moon_data_res_batch_norm-0.1}
	
\end{figure}

The results on the synthetic data on a 100-layer network are summarized in Fig. \ref{fig:moonaccuracy}. Reducing $h=0.1$ offers a net advantage in accuracy and variance of accuracy, with ResNet and batch norm ResNet both at or near an optimal 1.0.
\begin{figure}[h]
	\centering
	\includegraphics[scale = 0.5]{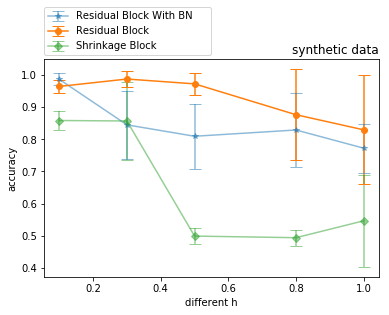}
	\caption{Different $h$ on the synthetic data using a 100-layer network. Mean and std. dev. of test accuracy over 10 trials.}
	\label{fig:moonaccuracy}
\end{figure}

\section{Experiments}

We now turn to larger-scale experiments to validate our hypotheses. First we make use of several  datasets from the UCI  repository \cite{Dua:2017}. Then, we solve a  problem  on the Bitcoin cryptocurrency transaction network,  an example of a time-evolving dataset. 
Experiments were implemented using  Keras  \cite{chollet2015keras} as well as a modified python3 library based on Keras.  
All experiments were run on a Linux, x$86\_64$ server with 64 Intel(R) \ Xeon(R)  \  CPU  E5  - $2683 \ v4 \ @ 2.10GHz$  and 5 GeForce \ GTX \ 1080 GPU. Batch sizes are not optimized. Experiment details and in the case of the BTC example the data itself are included in the supplementary materials.

\subsection{Experiments with the UCI repository}

We train and test a deep  network with transition blocks,  (\ref{res_transition_block}) (\ref{stable_transition_block}), using an increasing  explicit step size $h$ from $h=0.1$ to $h=1$. 
For each value of  $h$, we perform  10 trials and plot the mean and the standard deviation of the test accuracy. 

\textbf{Wine dataset}
This  dataset  provides the chemical analyses of 13 constituents found in Italian wines  from three different cultivars. The goal is to identify the type of wine.
The following hyperparameter settings are used: 50 epochs, 20 hidden blocks  with width of size 13 and batch size 10. 
\begin{figure}[h]
	\centering
	\includegraphics[scale = 0.5]{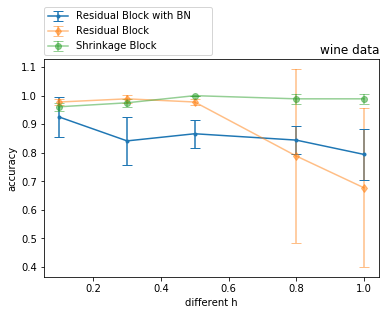}
	\caption{Different $h$ on the Wine dataset. Mean and std. dev. of test accuracy over 10 trials.}
	\label{fig:wine_data}
\end{figure}

\textbf{Abalone dataset}
The goal is to predict the size of the abalone from physical measurements rather than using the manual method of cutting the shell and counting the number of rings through a microscope. The problem is modeled as a 3-category classification problem: small (1-8 rings), medium (9-10 rings) and large (11-29 rings).
The   hyperparameter settings are: 50 epochs, 50 hidden blocks  with width of size 8 and batch size 50. 

\begin{figure}[h]
	\centering
	\includegraphics[scale = 0.5]{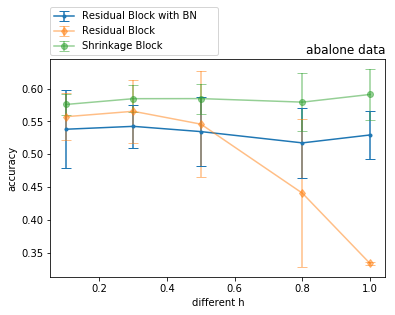}
	\caption{Different $h$ on Abalone dataset. Mean and std. dev. of test accuracy over 10 trials.}
	\label{fig:abalone_data}
\end{figure}

\textbf{Absenteeism dataset}
The goal is to predict absenteeism at work; the problem is modeled as a 4-category classification using a discretion of absentee time. The categories are   A : 0,  B: 1-16,  C: 17-56  and  D: greater than 56. 
The   hyperparameter settings are 20 epochs, 100 hidden blocks  with width of size 19 and batch size 10.

\begin{figure}[h]
	\centering
	\includegraphics[scale = 0.5]{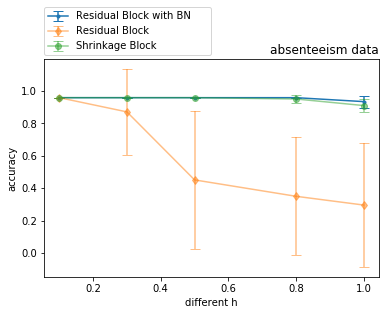}
	\caption{Different $h$ on the Absenteeism dataset. Mean and std. dev. of test accuracy over 10 trials.}
	\label{fig:abalone_data}
\end{figure}

\textbf{Autism screening adult dataset}
The goal of this binary classification problem is to predict whether an adult has been diagnosed with Autistic Spectrum Disorder (ASD). 
The   hyperparameter settings are 30 epochs, 100 hidden blocks  with width of size 20 and batch size 10. 

\begin{figure}[h]
	\centering
	\includegraphics[scale = 0.5]{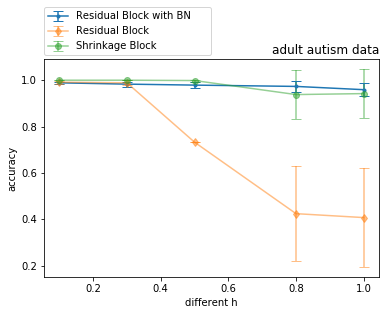}
	\caption{Different $h$ on the Autism Screening  dataset. Mean and std. dev. of test accuracy over 10 trials.}
	\label{fig:ASD_data}
\end{figure}

\textbf{Adult dataset}
The data was obtained from a US Census database of working adults and the goal is to predict whether an individual earns more than \$50K/year. 
The   hyperparameter settings are 50 epochs, 100 hidden blocks  with width of size 14 and batch size 100.

\begin{figure}[h]
	\centering
	\includegraphics[scale = 0.5]{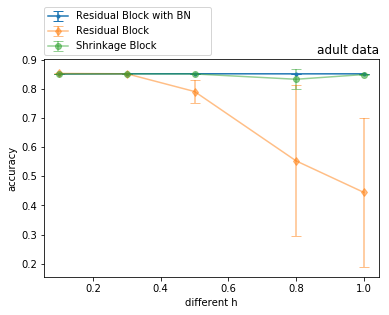}
	\caption{Different $h$ on the  Adult dataset. Mean and std. dev. of test accuracy over 10 trials.}
	\label{fig:adult_data}
\end{figure}

\subsubsection{Results}

Figures \ref{fig:wine_data}--\ref{fig:adult_data} illustrate the accuracy of the methods for increasing  step size $h$, namely  plain ResNet, called residual block, ResNet with batch normalization, and ResNet with the shrinkage block.

In all datsets,  when $h=0.8$ or $h=1$, the latter being  plain vanilla ResNet, the residual block exhibits reduced test accuracy and  a large variance in  accuracy,  wide enough to render the original ResNet ineffective on  these problems. 

Reducing the value of $h$ but making no other changes to the ResNet, that is \emph{without the use of block normalization}, achieves optimal or near-optimal accuracy  as compared to batch norm or the shrinkage block. Notice furthermore the dramatically reduced variance of the test accuracy when $h=0.1$. In the case of the abalone dataset, increasing the network depth improves accuracy of plain ResNet further  when using a small $h$.

Turning to ResNet with block normalization,  the accuracy when $h=1$ is better and much less volatile than plain vanilla ResNet, and ResNet with batch norm is less sensitive to a reduced $h$. However, 
the accuracy of ResNet with batch norm is not better than that of ResNet with $h=0.1$, even when maintaining a fixed 100-layer network for all examples. In addition, the variance of the accuracy is higher to that of plain ResNet with $h=0.1$. It should be noted that we did not optimize the batch size in these examples; ResNet with batch norm may be able to achieve better results with a different batch size. 

However, this suggests strongly that the original ResNet with an explicit step size parameter, with or without an increase in the depth of the network, is a valuable alternative to increasing the stability and the performance of ResNet.

\subsection{Bitcoin cryptocurrency transaction network}

We are interested in classification on datasets which often exhibit time-evolution in their features and responses. 
One such application is the classification of Bitcoin cryptocurrency transactions. In this example, the goal is to predict the entity type performing a transaction using   features   from the BTC transaction network. We postulate that convergence  stability may help in  transferability of the trained model over time-evolving data, of  interest to  predicting  Bitcoin transaction behavior.

The transaction data using the Bitcoin core is parsed  with Blocksci \cite{kalodner2017blocksci}. 
We generate four features over all but the  coinbase transactions from 1/2012--5/2018, which include the number of input and output addresses,  BTC volume, and transaction time stamp. Ground truth labels for the entity type as a function of address are obtained from walletexplorer.com. The four entity categories used here are: exchange (e), services (s), gambling (g), and mining pool (m). Labels are expanded using the common spending heuristic \cite{meiklejohn2013fistful}.
These procedures give rise to 16.5 million labelled transactions of the 233 million transactions.
This dataset and its illustrations on Bitcoin transaction patterns over time can be downloaded from  \href{https://www.dropbox.com/sh/majyu9d69r1y1ya/AAD-fJ-KmbK9-eHQRn_7d8PEa?dl=0}{the supplement materials}.

Fig. \ref{fig:BTC_h} shows the BTC transaction classification. The network structure was fixed at 128 hidden blocks, width 32, 100 epochs and a 100 K batch size.
Setting the step size to $h=0.1$ offers improved accuracy although the variance of ResNet is quite high. Overall on this  large, noisy dataset, ResNet with the shrinkage block and ResNet with batch norm   offer the best accuracy and lowest variance.

However, Fig. \ref{fig:BTC_weeks} shows  the trained model prediction over time increasingly far from the training data. While the results exhibit noise,  the plain ResNet with $h=0.1$ and the shrinkage block with $h=0.1$ are most readily transferable over time. Plain vanilla ResNet performance is not good, and  Resnet with batch norm, while better, is far from optimal, for any value of $h$.

\begin{figure}[h!]
	\centering
	\includegraphics[scale = 0.5]{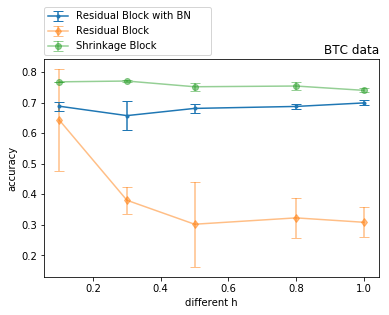}
	\caption{Different $h$ on  Bitcoin transaction entity classification. Mean and std. dev. of test accuracy over 5 trials.}
	\label{fig:BTC_h}
\end{figure}

\begin{figure}[h!]
	\centering
	\includegraphics[scale = 0.5]{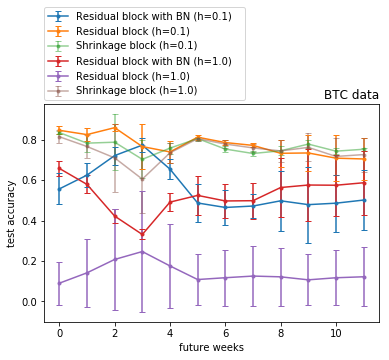}
	\caption{BTC  classification prediction accuracy over time.  Mean and std. dev. of test accuracy over 5 trials.}
	\label{fig:BTC_weeks}
\end{figure}

\section{Discussion and Next Steps}

Batch norm and the interpretation of ResNets as a step of the forward Euler equation for ODEs offer interesting avenues for improved accuracy, smoothness and stability of convergence.
We have shown visually how the  update blocks affect the convergence process on a 2-dim. binary classification problem. In particular,  use of an explicit step size, set to $h=0.1$, slows the convergence of Resnet but increases its smoothness and often its accuracy to an optimal level, even without increasing network depth. 

Batch norm has recently been  shown to increase the smoothness of the convergence process in a Lipshitz sense. We showed empirically that while indeed less volatile, batch norm on its own does not always lead to improved accuracy, but combined with a reduced step size (not necessarily as small as $h=0.1$)  is in some cases optimal. 

Batch norm requires more training time due to the many hyper-parameters; on the contrary, use of a small step size (without increasing depth) has no impact on training time. Thus, these two approaches can be seen as alternatives useful in achieving an optimal performance whilst improving stability of the convergence process. 

Several interesting avenues for future work can be suggested. While we used a fixed step size, nonlinear optimization suggests that an adaptive $h$ may be advantageous. A simple multi-layer approach, where depth was doubled and step size halved in successive layers, was defined by \cite{haber18b} but others,  less restrictive and potentially better-performing, can be studied. The more complex architectures suggested by the IVP interpretation of ResNet also provide considerable material for future research and performance improvements.

\clearpage

\bibliographystyle{aaai}
\bibliography{reference.bib}

\end{document}